%% file: ijcai26.tex
\documentclass{article}
\usepackage{ijcai26}

\usepackage{times}
\usepackage{soul}
\usepackage{url}
\usepackage[hidelinks]{hyperref}
\usepackage[utf8]{inputenc}
\usepackage[small]{caption}
\usepackage{graphicx}
\usepackage{amsmath}
\usepackage{amsthm}
\usepackage{booktabs}
\usepackage{algorithm}
\usepackage[switch]{lineno}
\usepackage{amssymb}
\usepackage{mathabx}
\usepackage{xcolor}
\usepackage{xspace}
\usepackage{algpseudocode}

\newtheorem{definition}{Definition}

\usepackage{color, colortbl}

\definecolor{Gray}{gray}{0.9}
\definecolor{fgreen}{RGB}{177,207,149}
\definecolor{fred}{RGB}{234,179,138}

\definecolor{firstcolor}{RGB}{20,128,85}
\definecolor{secondcolor}{RGB}{20,104,168}
\definecolor{thirdcolor}{RGB}{236,84,20}

\newcommand{\first}[1]{\textcolor{firstcolor}{\underline{\textbf{#1}}}}
\renewcommand{\second}[1]{\textcolor{secondcolor}{\underline{#1}}}
\renewcommand{\third}[1]{\textcolor{thirdcolor}{\underline{#1}}}

\newcommand{\method}{\textsc{TA-GGAD}\xspace}

\newtheorem{theorem}{Theorem}

\title{TA-GGAD: Testing-time Adaptive Graph Model for Generalist Graph Anomaly Detection}

\author{
Xiong Zhang${{\dagger}}$\and
Hong Peng${{\dagger}}$\and
Changlong Fu \and
Zhenli He \and
Xin	Jin \and
Yun	Yang \and
Cheng Xie$^*$
\affiliations
School of Software and AI, Yunnan University,  Kunming,  China\\
\emails
zhangxiong@stu.ynu.edu.cn, software\_ph@ynu.edu.cn, fuchanglong@stu.ynu.edu.cn, xinjin@ynu.edu.cn, yangyun@ynu.edu.cn, xiecheng@ynu.edu.cn,
}

\begin{document}

\maketitle

\begin{abstract}
A significant number of anomalous nodes in the real world, such as fake news, noncompliant users, malicious transactions, and malicious posts, severely compromises the health of the graph data ecosystem and urgently requires effective identification and processing.
With anomalies that span multiple data domains yet exhibit vast differences in features, cross-domain detection models face severe domain shift issues, which limit their generalizability across all domains.
This study identifies and quantitatively analyzes a specific feature mismatch pattern exhibited by domain shift in graph anomaly detection, which we define as the \emph{Anomaly Disassortativity} issue ($\mathcal{AD}$).
Based on the modeling of the issue $\mathcal{AD}$, we introduce a novel graph foundation model for anomaly detection. 
It achieves cross-domain generalization in different graphs, requiring only a single training phase to perform effectively across diverse domains. 
The experimental findings, based on fourteen diverse real-world graphs, confirm a breakthrough in the model's cross-domain adaptation, achieving a pioneering state-of-the-art (SOTA) level in terms of detection accuracy.
In summary, the proposed theory of $\mathcal{AD}$ provides a novel theoretical perspective and a practical route for future research in generalist graph anomaly detection (GGAD).  The code is available at \url{https://anonymous.4open.science/r/Anonymization-TA-GGAD/}.
\end{abstract}

\section{Introduction}
\label{sec:Introduction}

\begin{figure}
    \centering
    \includegraphics[width=0.98\linewidth]{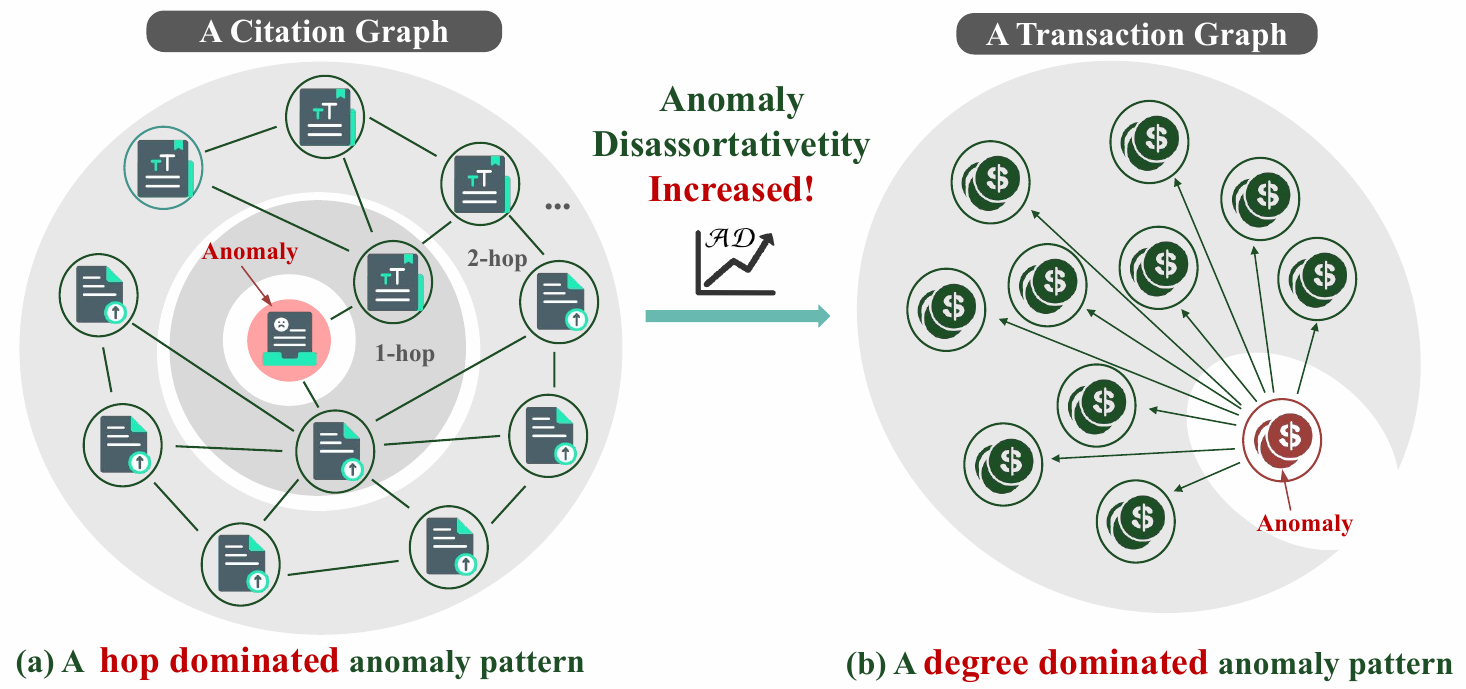}
    \caption{ A case of the Anomaly Disassortativity ($\mathcal{AD}$) issue in Generalist Graph Anomaly Detection. The anomaly pattern (a) is much different from (b).}
    \label{fig: motivation}
\end{figure}

Graph Anomaly Detection (GAD) plays a critical role in identifying suspicious or malicious nodes in complex relational networks, with applications in finance, e-commerce, cybersecurity, and social media~\cite{pourhabibi2020fraud,qiao2024deep,tang2022rethinking,li2022internet,pang2021deep,ma2021comprehensive,xia2021multi}. where anomalies often correspond to high-risk activities such as fraud, fake reviews, or compromised accounts. Recent graph neural network–based methods have achieved notable success~\cite{dominant_ding2019deep,cola_liu2021anomaly,GCTAM_zhang2025gctam,zhou2023robust}, yet most are developed under single-source domain settings and require retraining or fine-tuning when deployed on unseen graphs. This reliance limits their practicality in real-world environments characterized by diverse, dynamic, and continuously evolving networks.


To overcome this limitation, recent studies have introduced Generalized Graph Anomaly Detection (GGAD), which aims to develop a unified model capable of detecting anomalies across diverse graph domains without the need for target-domain retraining or fine-tuning~\cite{arc_liu2024arc,unprompt_niu2024zero,qiao2025anomalygfm}. By enabling zero-shot anomaly detection, GGAD represents a promising step toward scalable and practical solutions for large-scale, real-world graph environments.


ARC~\cite{arc_liu2024arc}, Unprompt~\cite{unprompt_niu2024zero}, and AnomalyGFM~\cite{qiao2025anomalygfm} advance generalizable graph anomaly detection through in-context learning, prompt-based representation unification, and graph-agnostic prototype pre-training, respectively. Despite these efforts, existing methods remain limited in handling fundamental cross-domain discrepancies. We attribute this limitation to a key underlying factor, termed \textit{Anomaly Disassortativity} ($\mathcal{AD}$), which captures the intrinsic variability of anomalous node behaviors across domains. Such inconsistency hampers the learning of transferable anomaly patterns, thereby constraining the effectiveness of current GGAD models in real-world cross-domain settings. Specifically, we identify two core types of disassortativity challenges:

\begin{itemize}
    \item \textbf{Node disassortativity ($\mathcal{ND}$)} refers to discrepancies in node feature distributions or semantics across domains. Nodes in different domains may vary in terms of their feature scale, dimensionality, or semantic meaning. This can result in a model trained in one domain to misinterpret node attributes in another. As illustrated in Figure~\ref{fig: motivation}, graph (a) represents citation data where node features are based on a bag-of-words representation, while graph (b) depicts financial transaction data with node features based on behaviour.

    \item \textbf{Structure disassortativity ($\mathcal{SD}$)} refers to the variability of graph connectivity patterns across domains. Differences in community structures, degree distributions, or overall topologies may cause patterns that appear normal in one domain to be perceived as anomalous in another, as illustrated in Figure~\ref{fig: motivation}. Specifically, (a) shows a \textbf{hop-dominated anomaly pattern} with irregular multi-hop connections, while (b) shows a \textbf{degree-dominated anomaly pattern} where nodes have abnormal degree distributions.
\end{itemize}

\textbf{Anomaly Disassortativity ($\mathcal{AD}$)} causes existing models to misinterpret node attributes or structural cues when transferred to unseen graphs, resulting in severe performance degradation. To address this challenge, we propose \method (Testing-time Adaptive Generalized Graph Anomaly Detection), a unified framework for cross-domain anomaly detection. \method mitigates anomaly disassortativity by jointly modeling node-level and structure-level irregularities through high-order and low-order anomaly scoring modules, producing node-aware and structure-aware anomaly scores, respectively. An anomaly disassortativity-aware adapter then leverages disassortativity indicators ($\mathcal{ND}$, $\mathcal{SD}$) to adaptively balance these scores. Furthermore, a testing-time adapter enables zero-shot adaptation via pseudo-label refinement during inference, allowing domain-specific adjustment without retraining or fine-tuning.

Extensive experiments on diverse graph datasets demonstrate that \method achieves SOTA performance and consistently outperforms existing GGAD baselines. Notably, \method improves AUROC by 15.73\% on the CS dataset compared with ARC~\cite{arc_liu2024arc}. \textbf{Our contributions are summarised as follows}: 




\begin{itemize}
    \item We \textbf{empirically identify and validate} the Anomaly Disassortativity ($\mathcal{AD}$) phenomenon through extensive cross-domain experiments, demonstrating its significant impact on generalized graph anomaly detection. 
    \item We provide a \textbf{formal theoretical definition} of the $\mathcal{AD}$ issue 
    and further design a quantitative metric to rigorously measure its degree across different domains, enabling systematic characterization and analysis.
    \item A novel \emph{test-time adaptive} framework, \textbf{TA-GGAD}, is proposed to dynamically adapt the anomaly detection process to mitigate $\mathcal{AD}$ issue, without requiring retraining or labeled target-domain data. 
\end{itemize}

Together, we identify $\mathcal{AD}$ as a core challenge and present a scalable solution for cross-domain graph anomaly detection.


\section{Related Work}
\subsection{Graph Anomaly Detection}

Graph anomaly detection~\cite{CAGAD_WWW_2025,FGAD_WWW_2025} aims to identify nodes or substructures in a graph that deviate from normal patterns~\cite{bwgnn_tang2022rethinking}. Traditional GAD methods are typically trained on a single dataset and can be categorized by supervision, semi-supervised, and unsupervised methods.

Supervised GAD relies on labeled anomalies or normal nodes, but obtaining reliable anomaly labels is often impractical~\cite{bgnn_ivanov2021boost,bwgnn_tang2022rethinking,ghrn_gao2023addressing,gcn_kipf2017semi,gat_velivckovic2018graph,CAGAD_xiao2024counterfactual,dong2025spacegnn}. Semi-supervised GAD assumes partial normal labels, exemplified by S-GAD~\cite{semi_qiao2024generative}, which synthesizes artificial outliers for one-class training. Unsupervised GAD operates without labels, typically leveraging autoencoding or one-class objectives, such as DOMINANT~\cite{dominant_ding2019deep}, CoLA~\cite{cola_liu2021anomaly}, and TAM-based methods~\cite{tam_qiao2024truncated,GCTAM_zhang2025gctam}. While effective on individual graphs, these approaches generally require retraining for each new graph, limiting cross-domain generalization.

\subsection{Generalist Graph Anomaly Detection}
To address the one-model-per-dataset limitation, recent studies have explored generalist graph anomaly detection. Early cross-domain GAD methods~\cite{ding2021cross,wang2023cross} rely on strong correlations between source and target graphs, which restricts their applicability. 

ARC~\cite{arc_liu2024arc} proposes a one-for-all GAD framework based on in-context learning, enabling adaptation to new graphs via feature alignment and few-shot normal references. UNPrompt~\cite{unprompt_niu2024zero} introduces a zero-shot generalist model that normalizes node attributes and leverages learned neighborhood prompts to produce a unified anomaly score across graphs. AnomalyGFM~\cite{qiao2025anomalygfm} further develops a GAD-oriented graph foundation model by aligning graph-agnostic normal and abnormal prototypes, supporting both zero-shot inference and few-shot prompt tuning.

Despite these advances, existing generalist methods remain limited. ARC requires target-domain samples at inference and cannot fully resolve domain shifts, UNPrompt assumes consistent attribute semantics and a single-source setting, and AnomalyGFM’s residual-based prototype alignment oversimplifies cross-domain anomaly semantics.


\section{Problem Statement}
\label{sec: problem}

\textbf{Preliminaries and Task Formulation.} 
An attributed graph is defined as $\mathcal{G} = (\mathcal{V}, \mathbf{A}, \mathbf{X})$, where $\mathcal{V}=\{v_1,\ldots,v_N\}$ denotes the node set, $\mathbf{X}\in\mathbb{R}^{N\times d}$ is the node feature matrix with $\mathbf{x}_i\in\mathbb{R}^d$ representing the feature vector of node $v_i$, and $\mathbf{A}\in\{0,1\}^{N\times N}$ is the adjacency matrix indicating graph connectivity. In the generalized graph anomaly detection (GGAD) setting, the objective is to learn a single anomaly detector from a collection of labeled source-domain graphs $\mathcal{T}_{\text{train}}=\{\mathcal{D}^{(1)}_{\text{train}},\ldots,\mathcal{D}^{(n_s)}_{\text{train}}\}$ and directly apply it to a disjoint set of target-domain graphs $\mathcal{T}_{\text{test}}=\{\mathcal{D}^{(1)}_{\text{test}},\ldots,\mathcal{D}^{(n_t)}_{\text{test}}\}$ without access to target-domain labels or retraining. Each dataset $\mathcal{D}^{(i)}=(\mathcal{G}^{(i)},\mathbf{y}^{(i)})$ corresponds to a graph from an arbitrary domain. Unlike conventional GAD methods that train graph-specific detectors, GGAD aims to construct a unified model capable of detecting anomalies across various graph domains with diverse semantics and structures.

\textbf{$\mathcal{AD}$ Problem Definition.}
The core challenge of GGAD arises from \emph{Anomaly Disassortativity} ($\mathcal{AD}$), which captures domain-level inconsistencies of anomalous patterns between source and target graphs. We formalize $\mathcal{AD}$ to quantify cross-domain discrepancies at both node and structural levels, reflecting their joint misalignment.

\begin{definition}
 Empirically, for source domain $S$  and target domain $T$, their $\mathcal{AD}$ can be quantified as a weighted combination of two components: \textbf{Node Disassortativity} ($\mathcal{ND}$) and \textbf{Structure Disassortativity} ($\mathcal{SD}$), which account for node-level and structure-level misalignments, respectively. Formally, we have 
\end{definition}

\begin{equation}
\label{eq: AD}
\mathop{\mathcal{AD}}\limits_{S \rightarrow T}=  \left|\mathop{\mathcal{ND}}\limits_{S \rightarrow T} -\mathop{\mathcal{SD}}\limits_{S \rightarrow T}\right|^{ \left(1+ \frac{\mathop{\mathcal{ND}}\limits_{S \rightarrow T}+\mathop{\mathcal{SD}}\limits_{S \rightarrow T}}{2}\right)} 
\end{equation}

To comprehensively characterize and quantify the impact of two parts, we propose two formal mathematical formulations that allow us to rigorously measure the degree of $\mathcal{ND}$ and $\mathcal{SD}$.
Supposing that we have node anomaly score sets from the source domain $S= \{s_1^S, \dots, s_{N_S}^S\}$ and the target domain $T = \{s_1^T, \dots, s_{N_T}^T\}$, for any node $v_i$ with anomaly score $x$, its probability density functions are estimated using Gaussian Kernel Density (KDE) estimation as

\begin{equation}
\label{eq:kde_source}
{f}^{\text{S}}(x) = \frac{1}{N_S h \sqrt{2\pi}} \sum_{i=1}^{N_S} \exp\!\left(-\frac{(x - s_i^S)^2}{2h^2}\right), \quad if \quad x \in S
\end{equation}

\begin{equation}
\label{eq:kde_target}
{f}^{\text{T}}(x) = \frac{1}{N_T h \sqrt{2\pi}} \sum_{j=1}^{N_T} \exp\!\left(-\frac{(x - s_j^T)^2}{2h^2}\right), \quad if \quad x \in T
\end{equation}

where $h$ is the bandwidth parameter controlling the smoothness.
In other words, for any given anomaly score $x$, ${f}^{\text{S}}{(x)}$ and ${f}^{\text{T}}{(x)}$ represent the estimated likelihood of observing this score in the source and target domains, respectively.


\textit{\textbf{(1) Node Disassortativity ($\mathcal{ND}$).}} 
We quantify node-level cross-domain discrepancy using the Jensen--Shannon (JS) distance between KDE-estimated anomaly score distributions. Let $\widehat{f}^{\text{S}}(x)$ and $\widehat{f}^{\text{T}}(x)$ denote the probability density functions of node-level anomaly scores in the source and target domains, respectively. The node disassortativity is defined as

\begin{equation}
\label{eq: ND}
\mathop{\mathcal{ND}}\limits_{S \rightarrow T}
= \sqrt{\tfrac{1}{2} D_{\mathrm{KL}}\!\Big(\widehat{f}^{\text{S}}(x) \,\big\|\, M^{\text{node}} \Big) 
+ \tfrac{1}{2} D_{\mathrm{KL}}\!\Big(\widehat{f}^{\text{T}}(x) \,\big\|\, M^{\text{node}} \Big)},
\end{equation}
where $M^{\text{node}}=\tfrac{1}{2}\big(\widehat{f}^{\text{S}}(x)+\widehat{f}^{\text{T}}(x)\big)$ is the mixture distribution, and $D_{\mathrm{KL}}(\cdot\|\cdot)$ denotes the Kullback--Leibler divergence. By construction, $\mathcal{ND}\in[0,1]$, and larger values indicate stronger distributional mismatch between the source and target domains under the node disassortativity mechanism.


\textit{\textbf{(2) Structure Disassortativity ($\mathcal{SD}$).}} 
We characterize structure-level disassortativity via distributional differences in structure-derived anomaly scores between source and target domains. Let $\widebar{f}^{\text{S}}(x)$ and $\widebar{f}^{\text{T}}(x)$ denote the Gaussian KDE distributions of structure-level anomaly scores for the source and target domains, respectively. The structure disassortativity is defined as

\begin{equation}
\label{eq: SD}
\mathop{\mathcal{SD}}\limits_{S \rightarrow T}
= \sqrt{\tfrac{1}{2} D_{\mathrm{KL}}\!\Big(\widebar{f}^{\text{S}}(x) \,\big\|\, M^{\text{str}} \Big) 
+ \tfrac{1}{2} D_{\mathrm{KL}}\!\Big(\widebar{f}^{\text{T}}(x) \,\big\|\, M^{\text{str}} \Big)},
\end{equation}

where $M^{\text{str}} = \tfrac{1}{2}\big(\widebar{f}^{\text{S}}(x)+\widebar{f}^{\text{T}}(x)\big)$ is the mixture distribution. 
It is easy to notice that $\mathcal{SD}\in[0,1]$, and larger values indicate stronger distributional mismatch between the source and target domains under the structural disassortativity mechanism. 

\textbf{Empirical Evidence.} Quantitative analyses of $\mathcal{ND}$ and $\mathcal{SD}$ are presented in Sections~\ref{sec: analysis_of_AD} and~\ref{sec: analysis ND and SD}.


\begin{figure*}
    \centering
    \includegraphics[width=0.98\linewidth]{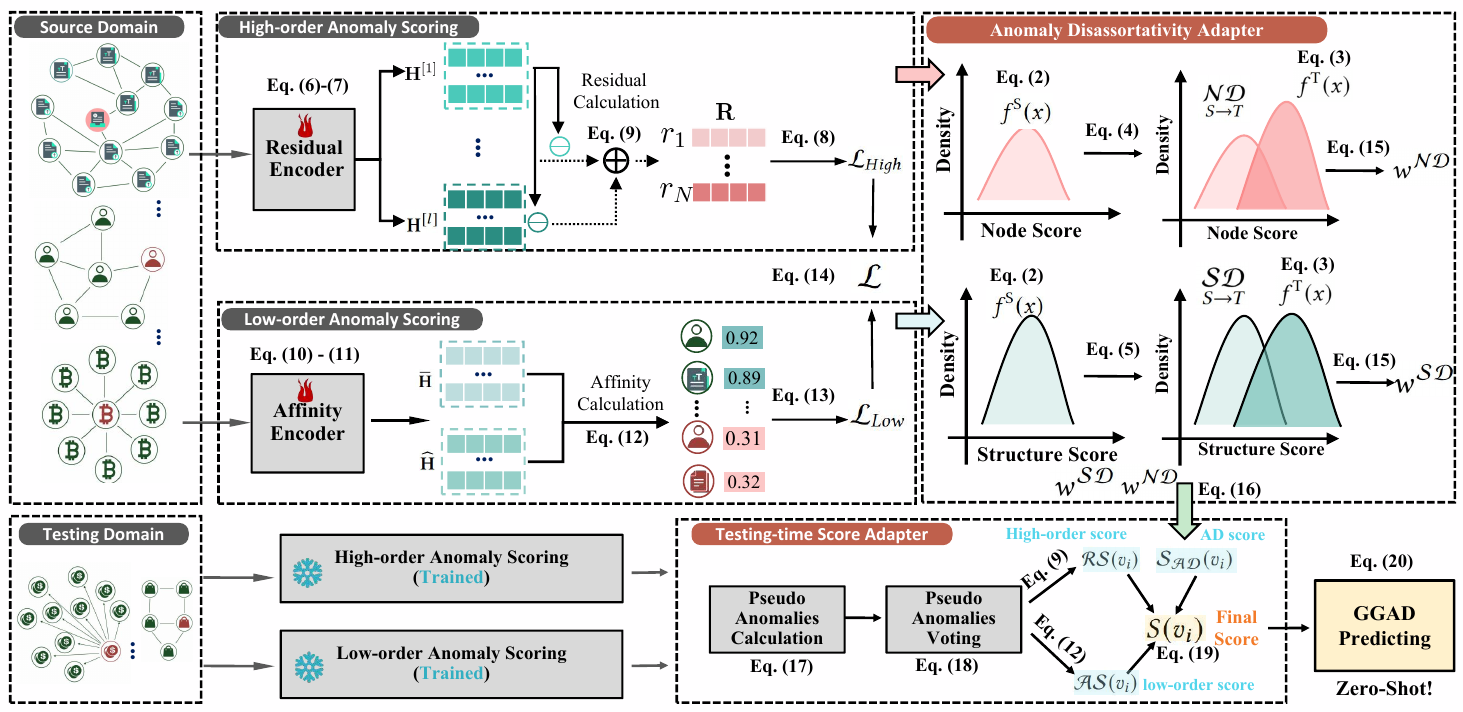}
    \caption{Overview of \method}
    \label{fig: framework}
\end{figure*}

\section{The Proposed Method}
In this section, we introduce \method, a unified framework for addressing cross-domain anomaly disassortativity in graphs. As illustrated in Figure~\ref{fig: framework}, \method comprises four modules: 

\textbf{(1) High-order Anomaly Scoring}: Captures node-level attribute deviations via high-order feature dependencies, producing a \textbf{node-aware score}.
\textbf{(2) Low-order Anomaly Scoring}: Models structure-level irregularities through topological affinity, yielding a \textbf{structure-aware score}.
\textbf{(3) Anomaly Disassortativity-Aware Adapter}: Mitigates cross-domain bias by adaptively weighting scores using $\mathcal{ND}$ and $\mathcal{SD}$.
\textbf{(4) Testing-time Adapter}: Enables zero-shot adaptation via pseudo-label refinement during inference without retraining.

\subsection{High-order Anomaly Scoring}
\label{subsec: 4.1}
Conventional GNN methods rely on shallow message passing, typically aggregating only 1-hop neighborhood information, which may fail to capture high-order patterns essential for anomaly detection. 

\textbf{High-order Residual Representation.}
Learning high-order representations directly for anomalies on graphs will lead to an anomaly-normal mixture problem, as multi-hop propagation reduces the discriminability between the representations of normal and anomalous nodes.
In this paper, we abandon the direct representation of normal and anomalous nodes and instead represent their high-order residuals, which can effectively reduce the confusion in high-order representations between normal and anomalous nodes. 

\begin{equation}
\label{eq: X_prop_trans}
\begin{aligned}
\mathbf{H}^{[1]} &= \sigma(\mathbf{D}^{-\frac12}\tilde{\mathbf{A}}\mathbf{D}^{-\frac12}\mathbf{X} \mathbf{W}^{[1]}),\\ 
\mathbf{H}^{[2]} &= \sigma(\mathbf{D}^{-\frac12}\tilde{\mathbf{A}}\mathbf{D}^{-\frac12}\mathbf{H}^{[1]} \mathbf{W}^{[2]}),\\ 
\cdots\\
\mathbf{H}^{[l]} &= \sigma(\mathbf{D}^{-\frac12}\tilde{\mathbf{A}}\mathbf{D}^{-\frac12}\mathbf{H}^{[l-1]} \mathbf{W}^{[l]}),
\end{aligned}
\end{equation}

where $\mathbf{X}$ denotes node feature matrix, $\tilde{\mathbf{A}}=\mathbf{A}+\mathbf{I}$ is the adjacency matrix with self-loops added, $\mathbf{W}^{(l)}$ and $\sigma(\cdot)$ represent the learnable weight matrix and nonlinear
activation function at the $l$-th layer, respectively.
$\mathbf{H}^{(l)}$ denotes the node representations after $l$-hop propagation.

Then, we calculate the residuals between each hop representation to get the residual representation for the graph, as shown in the equation (\ref{eq: residual_information}):
\begin{equation}
\label{eq: residual_information}
\begin{aligned}
\mathbf{R} &= [(\mathbf{H}^{[2]} - \mathbf{H}^{[1]}) ||  \cdots ||(\mathbf{H}^{[l]} - \mathbf{H}^{[1]})] =            \begin{bmatrix}
\mathbf{r}_1, \cdots,
\mathbf{r}_N
\end{bmatrix}  
\end{aligned}
\end{equation}
where $\mathbf{R} \in \mathbb{R}^{N \times ((l-1)d_h)}$ denotes the residual embedding matrix for all nodes. $\mathbf{r}_i$ is the residual representation for the node $v_i$. This residual representation captures the differences between multi-hop propagated features and the initial node embeddings, emphasizing anomalous patterns while preserving information from normal nodes.

\textbf{Residual contrastive loss.}
We adopt a contrastive learning objective that pulls normal nodes together while pushing them away from anomalous nodes. Specifically, similarities between normal–normal pairs are maximized, whereas a margin-based constraint is imposed on normal–anomalous pairs to enforce separation. The residual similarity loss is defined as

\begin{equation}
\begin{aligned}
     \mathcal{L}_{High} = \sum_t^{N_S^+} \sum_i^{N_S^+}  (1-\frac{\mathbf{r}^+_t \cdot \mathbf{r}^+_i}{||\mathbf{r}^+_t||\cdot||\mathbf{r}^+_i||})  + \\
     \sum_t^{N_S^+}\sum_j^{N_S^-} max(0 , \frac{\mathbf{r}^+_t \cdot \mathbf{r}^-_j}{||\mathbf{r}^+_t||\cdot||\mathbf{r}^-_j||}-\epsilon),
\end{aligned}
\label{eq: loss_high}
\end{equation}

where $\epsilon$ is a margin hyperparameter, and $\mathbf{r}_t^+$ and $\mathbf{r}_j^-$ denote the residual representations of normal and anomalous nodes, respectively.




\textbf{Residual Scoring ($\mathcal{RS}$).}
After training, we evaluate node-level deviations in the target domain using residual embeddings $\mathbf{r}_i$ (Eq.~\eqref{eq: residual_information}), which capture discrepancies induced by node disassortativity. The residual score is defined as the mean squared distance between $\mathbf{r}_i$ and $n_k$ randomly sampled residuals from the target graph:
\begin{equation}
\label{eq: RS_score}
\mathcal{RS}(v_i)=\frac{1}{n_k}\sum_{j=1}^{n_k}\|\mathbf{r}_i-\mathbf{r}_j\|^2,
\end{equation}
where $\mathbf{r}_j$ denotes residual embeddings extracted by the frozen encoder. A higher $\mathcal{RS}$ values indicate stronger node-level anomalies in the target domain.

\subsection{Low-order Anomaly Scoring}
\label{subsec: 4.2}
Although multi-hop propagation enriches high-order contextual representation, it is insufficient for capturing domain-specific distinctions; thus, we retain low-order information to better measure and preserve structural and semantic discrepancies between the source and target domains.

\textbf{Affinity Encoder.}  
To explicitly cope with the structural modeling problem, we introduce a structural graph affinity encoder that learns a homophily‑driven local affinity score for graphs within source and target domains.  
Normal nodes are expected to exhibit high affinity with their neighbours, whereas anomaly nodes break this pattern and thus receive low affinity. 
This contrast provides a structure-aware signal that complements the high-order representations described in Section~\ref{subsec: 4.1}.

Firstly, we map each node to a latent space that captures neighbourhood context information. For computational efficiency and fair comparison with prior work, we employ a single graph convolution network (GCN) layer~\cite{gcn_kipf2017semi}:
\begin{equation}
\begin{aligned}
\label{eq: gcn_structure_information}
    \widebar{\mathbf{H}} = \sigma( {{{\bf{D}}^{ - \frac{1}{2}}}{\bf{ A}}{{\bf{D}}^{ - \frac{1}{2}}}{{{\mathbf{X}}}}{{\bf{W}}}} ) =\{ \mathbf{\bar{h}}_{1}, \cdots, \mathbf{\bar{h}}_{N}\},
\end{aligned}
\end{equation}
where $\sigma(\cdot)$ an activation function and $\mathbf{W}$ is the learnable weight matrix. Secondly, we apply a two-layer MLP to re-aggregate the centrality-related information of each node:
\begin{equation}
\begin{aligned}
\label{eq: mlp_structure_information}
    \widehat{\mathbf{H}} = \sigma(\mathbf{W_2}({{\mathbf{X}}}{{\bf{W_1}}})) =\{ \widehat{\mathbf{h}}_{1}, \cdots, \widehat{\mathbf{h}}_{N}\},
\end{aligned}
\end{equation}
In conclusion, we employ two projections to introduce nonlinearity and a controlled bottleneck, enhancing node-centric signals and yielding more expressive low-order embeddings.

\textbf{Affinity Scoring.}
Lastly, we compute low-order embeddings $\widebar{\mathbf{H}}$ and $\widehat{\mathbf{H}}$ to measure the conformity of node $v_i$ to its neighborhood $\mathcal{N}(v_i)$ via \emph{average cosine similarity}:
\begin{equation}
\label{eq: affinity score}
\mathcal{AS}(v_i)=
\frac{1}{|\mathcal{N}(v_i)|}
\sum_{v_j\in\mathcal{N}(v_i)}
\left(
\frac{\widebar{\mathbf{h}}_i\cdot\widebar{\mathbf{h}}_j}
{\|\widebar{\mathbf{h}}_i\|\|\widebar{\mathbf{h}}_j\|}
+
\frac{\widehat{\mathbf{h}}_i\cdot\widehat{\mathbf{h}}_j}
{\|\widehat{\mathbf{h}}_i\|\|\widehat{\mathbf{h}}_j\|}
\right).
\end{equation}
A higher $\mathcal{AS}(v_i)$ values indicate stronger homophily, while lower values suggest structural anomalies.

\textbf{Affinity Maximization Loss.}
We train the model by maximizing local affinity in an unsupervised manner:
\begin{equation}
\label{eq: loss_low}
\mathcal{L}_{\text{Low}}
= -\sum_{i=1}^{N}\mathcal{AS}(v_i),
\end{equation}
where $\Theta=[\mathbf{W},\mathbf{W}_1,\mathbf{W}_2]$ denotes the learnable parameters.

\textbf{Final Optimization Loss.}
To jointly optimize node-level and structure-level anomaly detection by combining the high-order residual loss $\mathcal{L}_{\text{High}}$ and the low-order affinity loss $\mathcal{L}_{\text{Low}}$:
\begin{equation}
\label{eq: loss_final}
\mathcal{L}=\mathcal{L}_{\text{Low}}+\mathcal{L}_{\text{High}}.
\end{equation}
This unified objective enables the model to capture both high-order anomalies and low-order irregularities, allowing for robust anomaly detection across graphs from heterogeneous domains.

\subsection{Anomaly Disassortativity Adapter}
\label{subsec: 4.3}
To mitigate distributional discrepancies in cross-domain anomaly detection, we propose the \textbf{Anomaly Disassortativity Adapter (ADA)}, which adaptively fuses two anomaly score channels through a weighted integration mechanism.  Specifically, the ADA selector aggregates residual-based and affinity-based anomaly scores to comprehensively capture both high-order and low-order distributional variations across domains.
Instead of explicit divergence estimation, ADA leverages node-level and structure-level disassortativity measures, $\mathcal{ND}$ and $\mathcal{SD}$, to guide
cross-domain weighting. Channels with smaller disassortativity indicate better alignment and are thus assigned larger weights. These weights are computed as follows:
\begin{equation}
\label{eq: adas_weight}
\begin{aligned}
    {w^{\mathcal{ND}}} = \frac{1}{(\mathop{\mathcal{ND}}\limits_{S \rightarrow T} + \epsilon)^{\tau}}, \quad 
    w^{\mathcal{SD}} = \frac{1}{(\mathop{\mathcal{SD}}\limits_{S \rightarrow T} + \epsilon)^{\tau}},
\end{aligned}
\end{equation}
where $\epsilon$ is a small constant for numerical stability, and $\tau$ controls the weighting sharpness, with larger $\tau$ emphasizing better-aligned channels and smaller $\tau$ yielding more uniform weights.

Finally, the fused anomaly score is computed as a weighted aggregation of residual-based and affinity-based anomaly scores:
\begin{equation}
\label{eq: fused_score}
\mathcal{S}_{\mathcal{AD}}(v_i) = \Big(w^{\mathcal{ND}} \cdot \mathcal{RS}(v_i) + w^{\mathcal{SD}} \cdot \Big(\textbf{1}-\mathcal{AS}(v_i)\Big)\Big),
\end{equation}
where $\mathcal{RS}(v_i)$ and $\mathcal{AS}(v_i)$ denote residual-based and affinity-based scores, respectively. This adaptive fusion enhances robustness by emphasizing well-aligned anomaly channels across domains.

\subsection{Testing-time Score Adapter}  
\label{subsec: 4.4}

At this stage, three types of anomaly scores have been obtained: the \textit{residual Score} $\mathcal{RS}(v_i)$ from High-order Anomaly Scoring, the \textit{affinity Score} $\mathcal{AS}(v_i)$ from Low-order Anomaly Scoring, and the \textit{fused score} $\mathcal{S}_{\mathcal{AD}}(v_i)$ from the Anomaly Disassortativity-Aware adapter.  Each of these provides a distinct view of node abnormality across domains.

Let $\mathbf{S}^{(k)}=\{s^{(k)}_1,\ldots,s^{(k)}_N\}$ denote the $k$-th anomaly score vector, with $k\in\{\mathcal{RS},\mathcal{AS},\mathcal{S}_{\mathcal{AD}}\}$ and $N$ is the number of nodes.
Each vector is sorted in descending order, and the top-$M$ entries are assigned as
pseudo-anomalies:
\begin{equation}
\label{eq: pseudo_label}
\widehat{y}^{(k)}_i=
\begin{cases}
1, & \text{if } i\in \text{Top-}M(\mathbf{S}^{(k)}),\\
0, & \text{otherwise},
\end{cases}
\end{equation}
where $\widehat{y}^{(k)}_i \in \{0,1\}$ and $\text{Top-}M(\cdot)$ returns the index set of the top-$M$ anomaly scores. Here, $M$ is determined by the anomaly ratio of the dataset.

To mitigate noise from individual score sources, we employ a majority-voting strategy to integrate pseudo-labels:  
\begin{equation}
\label{eq: voting}
\widehat{y}_i=
\begin{cases}
1, & \sum_{k=1}^{K}\widehat{y}^{(k)}_i \ge K,\\
0, & \text{otherwise},
\end{cases}
\end{equation}
where $K$ denotes the number of voting score threshold, and $\widehat{y}_i$ is the final pseudo-label for node $v_i$. In practice, three score sources are used, and their consensus integrates node-level, structure-level, and domain-level evidence to provide robust pseudo-supervision for testing-time adaptation.

Finally, we introduce an adaptive weighting scheme to account for varying score reliability. Let $\mathbf{S}=[\mathbf{s}^{(1)},\ldots,\mathbf{s}^{(K)}]\in\mathbb{R}^{N\times K}$ denote the
anomaly score matrix. The final testing-time score is
\begin{equation}
\label{eq: adaptive_score}
S(v_i)=\sum_{k=1}^{K} w_k\, s_i^{(k)}, \quad
\text{s.t. } w_k\ge0,\ \sum_{k=1}^{K}w_k=1,
\end{equation}
where $\mathbf{w}$ are learnable reliability weights. They are optimized on pseudo-labeled nodes by
\begin{equation}
\label{eq: adaptive loss}
\min_{\mathbf{w}}\frac{1}{|\mathcal{I}|}\sum_{i\in\mathcal{I}}
\ell\!\left(S(v_i),\widehat{y}_i\right),
\end{equation}
with $\mathcal{I}$ being the set of pseudo-labeled nodes and $\ell(\cdot,\cdot)$ a weakly supervised loss (e.g., cross-entropy).  

Through pseudo-label voting and adaptive weighting, the Testing-time Adapter dynamically emphasizes informative anomaly cues, including $\mathcal{RS}$, $\mathcal{AS}$, and $\mathcal{S}_{\mathcal{AD}}$, enabling \textbf{zero-shot adaptation} to unseen target domains. Detailed algorithmic details and complexity analysis are provided in Appendix~\ref{appendix: algo_des}.

\begin{table*}[ht!]
\renewcommand{\arraystretch}{1.2}
\setlength{\extrarowheight}{3pt}     
\centering
\caption{Anomaly detection performance in terms of AUROC (in percent, mean±std).  The results ranked \textcolor{firstcolor}{\textbf{\underline{first}}}, \textcolor{secondcolor}{\underline{second}}, and \textcolor{thirdcolor}{\underline{third}}.} 
\label{tab: main_auroc_sorted}
\resizebox{\textwidth}{!}{%
\begin{tabular}{c|ccccccccccccc|c}
\hline
\textbf{Method}  & \textbf{CS} & \textbf{Facebook} & \textbf{ACM} & \textbf{Cora} & \textbf{T-Finance} & \textbf{CiteSeer} & \textbf{BlogCatalog} & \textbf{elliptic} & \textbf{Amazon} & \textbf{DGraph-Fin} & \textbf{photo} & \textbf{Reddit} & \textbf{Weibo} & \textbf{rank} \\ \hline

\rowcolor{Gray} \multicolumn{15}{c}{\textbf{GAD methods}} \\

GCN(2017)        & 67.55±1.11 & 29.51±4.86 & 60.49±9.65 & 59.64±8.30 & 68.05±11.82 & 60.27±8.11 & 56.19±6.39 & 50.67±0.49 & 46.63±3.47 & 40.21±2.01 & 58.96±0.69 & 50.43±4.41 & 76.64±17.69 & 9.61 \\
GAT(2018)        & 53.38±0.39 & 51.88±2.16 & 48.79±2.73 & 50.06±2.65 & 48.56±3.21 & 51.59±3.49 & 50.40±2.80 & 24.80±0.48 & 50.52±17.22 & OOM & 52.34±0.10 & 51.78±4.04 & 53.06±7.48 & 12.69 \\
BGNN(2021)       & 36.06±0.92 & 54.74±25.29 & 44.00±13.69 & 42.45±11.57 & 37.05±1.98 & 42.32±11.82 & 47.67±8.52 & 55.29±4.25 & 52.26±3.31 & OOM & 41.86±0.25 & 50.27±3.84 & 32.75±35.35 & 14.69 \\
BWGNN(2022)      & 61.41±10.21 & 45.84±4.97 & 67.59±0.70 & 54.06±3.27 & \third{70.36±4.16} & 52.61±2.88 & 56.34±1.21 & \second{72.41±4.91} & 55.26±16.95 & 52.43±2.48 & 62.47±4.81 & 48.97±5.74 & 53.38±1.61 & 8.85 \\
GHRN(2023)       & 61.50±0.04 & 44.81±8.06 & 55.65±6.37 & 59.89±6.57 & 61.14±15.68 & 56.04±9.19 & 57.64±3.48 & 57.02±4.65 & 49.48±17.13 & \third{52.75±1.11} & 56.22±4.59 & 46.22±2.33 & 51.87±14.18 & 10.08 \\
DOMINANT(2019)   & \second{82.56±0.03} & 51.01±0.78 & 70.08±2.34 & \third{66.53±1.15} & 59.89±0.27 & 69.47±2.02 & \third{74.25±0.65} & 24.80±0.48 & 48.94±2.69 & OOM & 60.18±0.01 & 50.05±4.92 & \first{92.88±0.32} & 7.69 \\
CoLA(2021)       & 58.61±1.64 & 12.99±11.68 & 66.85±4.43 & 63.29±8.88 & 51.34±1.16 & 62.84±9.52 & 50.04±3.25 & 38.72±8.46 & 47.40±7.97 & OOM & 41.32±2.48 & 52.81±6.69 & 16.27±5.64 & 12.30 \\
HCM-A(2022)      & 58.66±5.60 & 35.44±13.97 & 53.70±4.64 & 54.28±4.73 & OOM & 48.12±6.80 & 55.31±0.57 & OOM & 43.99±0.72 & OOM & 56.44±12.93 & 48.79±2.75 & 65.52±12.58 & 13.62 \\
TAM(2023)        & 70.41±2.17 & 65.88±6.66 & 74.43±1.59 & 62.02±2.39 & 39.42±10.61 & \third{72.27±0.83} & 49.86±0.73 & 22.31±0.79 & 56.06±2.19 & OOM & 60.71±1.31 & 55.43±0.33 & 71.54±0.18 & 7.85 \\
CAGAD(2024)      & 60.27±9.53 & 45.84±4.97 & 39.80±9.91 & 50.11±3.41 & \second{72.73±6.01} & 40.13±5.41 & 49.84±12.37 & \third{70.33±6.01} & 46.06±0.75 & 51.39±1.39 & 60.36±3.81 & 54.57±3.89 & 58.99±3.42 & 11.23 \\
S-GAD(2024)       & 50.78±3.24 & 55.89±8.99 & 37.47±2.68 & 39.44±5.41 & 47.15±2.47 & 38.18±4.21 & 50.70±7.34 & 32.88±8.13 & 53.11±4.92 & OOM & 30.77±7.68 & 55.39±0.44 & 65.73±3.35 & 13.15 \\
CONSISGAD(2024)  & 66.48±6.45 & 32.32±11.35 & 67.26±7.47 & 52.94±3.82 & 53.56±5.85 & 57.56±3.02 & 67.41±5.73 & 48.36±18.31 & \third{76.42±9.98} & 51.76±4.87 & 51.72±8.51 & 52.34±3.56 & 44.38±13.81 & 9.84 \\
SmoothGNN(2025)  & 45.19±10.47 & 52.69±9.60 & 47.06±8.23 & 51.50±5.38 & 31.66±27.47 & 48.04±5.79 & 47.96±6.01 & 61.60±9.52 & 62.25±21.10 & 52.71±7.59 & 44.02±15.27 & 51.31±2.73 & 24.55±15.81 & 12.46 \\
SpaceGNN(2025)   & 43.40±3.01 & 35.40±2.39 & 52.78±0.45 & 43.58±4.14 & OOM & 44.70±3.52 & 51.02±7.46 & 60.41±3.08 & 29.06±3.67 & \second{53.01±1.07} & 43.46±5.13 & 50.51±4.86 & 55.56±13.26 & 13.62 \\
GCTAM(2025)      & 74.12±1.62 & \second{68.61±1.41} & \second{80.14±0.13} & 58.78±2.17 & 40.74±14.37 & 70.31±1.77 & 67.60±0.77 & 24.56±1.37 & 55.74±0.60 & OOM & \third{63.12±0.31} & 59.32±0.73 & 70.61±0.10 & \third{6.61} \\ \hline

\rowcolor{Gray} \multicolumn{15}{c}{\textbf{GGAD methods}} \\

ARC(2024)        & \third{82.46±0.45} & \third{67.56±1.60} & \third{79.88±0.28} & \second{87.45±0.74} & 64.10±0.42 & \second{90.95±0.59} & \second{74.76±0.06} & 26.40±0.73 & \second{80.67±1.81} & 47.46±1.42 & \second{75.42±0.70} & \first{60.04±0.69} & \third{88.85±0.14} & \second{3.92} \\
UNPrompt(2025)  & 71.08±0.68 & 55.27±6.90 & 69.91±1.28 & 54.31±1.50 & 25.93±0.30 & 49.80±3.12 & 68.36±0.40 & 41.76±5.39 & 56.02±11.69 & 52.69±0.74 & 51.91±1.62 & 59.18±1.44 & 45.56±3.75 & 8.85 \\
AnomalyGFM(2025) & 67.86±4.95 & 58.64±7.14 & 60.79±1.48 & 54.17±3.08 & 67.57±3.72 & 54.71±2.30 & 57.77±3.31 & OOM & 60.65±9.07 & OOM & 51.57±0.67 & \second{59.99±1.69} & 69.48±11.11 & 8.61 \\ \hline

\method         & \first{98.19±0.13} & \first{83.39±2.15} & \first{89.04±0.76} & \first{91.07±0.91} & \first{76.29±1.21} & \first{94.12±0.48} & \first{77.91±0.36} & \first{74.77±0.14} & \first{82.67±1.37} & \first{54.47±2.82} & \first{76.10±1.49} & \third{59.69±2.50} & \second{90.58±0.17} & \first{1.23} \\ \hline
$\Delta$         & \textcolor{red}{$\uparrow$15.73} & \textcolor{red}{$\uparrow$14.78} & \textcolor{red}{$\uparrow$8.90} & \textcolor{red}{$\uparrow$3.62} & \textcolor{red}{$\uparrow$3.56} & \textcolor{red}{$\uparrow$3.17} & \textcolor{red}{$\uparrow$3.15} & \textcolor{red}{$\uparrow$2.36} & \textcolor{red}{$\uparrow$2.00} & \textcolor{red}{$\uparrow$1.46} & \textcolor{red}{$\uparrow$0.68} & \textcolor{green}{$\downarrow$0.35} & \textcolor{green}{$\downarrow$2.30} & \\ \hline
\end{tabular}%
}
\end{table*}

\section{Experiments}
\subsection{Experimental Setup}

\textbf{Dataset configuration.}  Following ARC~\cite{arc_liu2024arc}, we create a deliberately shifted source/target split.  
The training set is $\mathcal{T}_{\text{train}}
 =$\{PubMed, Flickr, Questions, YelpChi\}, 
while the zero-shot test set is $\mathcal{T}_{\text{test}}
 =$\{ACM,  Facebook,  Amazon, Cora,  CiteSeer, BlogCatalog, Reddit, Weibo, CS, Photo, Elliptic, T-Finance, DGraph-Fin \}.  
All graphs contain injected or naturally occurring anomalies \cite{tang2024gadbench,cola_liu2021anomaly}, providing a realistic benchmark for generalization. More detailed descriptions of the datasets can be found in the appendix~\ref{appendix: dataDes}.

\textbf{Baselines.} We compare \method\ with sixteen strong competitors: five \emph{supervised} GNNs (GCN \cite{gcn_kipf2017semi}, GAT \cite{gat_velivckovic2018graph}, BGNN \cite{bgnn_ivanov2021boost}, BWGNN \cite{bwgnn_tang2022rethinking}, GHRN \cite{ghrn_gao2023addressing}, CAGAD \cite{CAGAD_xiao2024counterfactual}, CONSISGAD\cite{ICLR2024_CONSISGAD}, SpaceGNN\cite{dong2025spacegnn}), one {semi-supervised} model (S-GAD \cite{semi_qiao2024generative}), six \emph{unsupervised} methods (DOMINANT \cite{dominant_ding2019deep}, CoLA \cite{cola_liu2021anomaly}, HCM-A \cite{huang2022hop}, TAM \cite{tam_qiao2024truncated}, SmoothGNN\cite{dong2025smoothgnn}, GCTAM \cite{GCTAM_zhang2025gctam}), and the three state-of-the-art \emph{generalist} approaches ARC \cite{arc_liu2024arc} and UNPrompt \cite{unprompt_niu2024zero} AnomalyGFM \cite{qiao2025anomalygfm}.

\textbf{Evaluation protocol.} We report AUROC and AUPRC averaged over five random seeds (mean$\pm$std) \cite{tang2024gadbench,pang2021toward}. All methods are trained once on $\mathcal{T}_{\text{train}}$ and evaluated on each target graph $\mathcal{T}_{\text{test}}$ in a \emph{pretrain-only} setting. Feature dimensions are aligned via learnable or random projections, and hyperparameters are selected on the training split and fixed across all targets, with no dataset-specific tuning. Implementation details are provided in Appendix~\ref{appendix: implementation}.

\subsection{Main Results}

Table~\ref{tab: main_auroc_sorted} reports AUROC results on eighteen benchmarks. 
\method ranks first on \textbf{11/13 key datasets}, with notable gains on CS (+15.73\%), Facebook (+14.78\%), and ACM (+8.90\%), validating the benefit of jointly modeling $\mathcal{ND}$ and $\mathcal{SD}$ for zero-shot generalization. 
Consistent improvements are also observed on Amazon, Cora, and CiteSeer. 
An exception is Weibo, where DOMINANT slightly outperforms \method, likely due to strong favoring reconstruction-based detection. Overall, \method maintains a \textbf{mean rank of 1.23} and low standard deviations ($<$ 3\%) across random seeds, highlighting stable and consistently superior performance.  
AUPRC results are provided in Appendix~\ref{tab: main_aurpc}.

\subsection{Ablation Study}
We evaluate three variants to assess key component contributions:
(i) \textbf{w/o ADA\&TSA} (backbone only),
(ii) \textbf{w/ADA} (with Anomaly Disassortativity Adapter),
(iii) \textbf{w/TSA} (with Testing-time Score Adapter).

\begin{table}[h!]
\centering
\caption{The evaluation of anomaly disassortativity Adapter (ADA), Testing-time Score Adapter (TSA), and \method (full).}
\label{tab: ablation}
\resizebox{\linewidth}{!}{%
\begin{tabular}{c|cccc}
\hline
\textbf{Dataset} & \textbf{w/o ADA\&TSA} & \textbf{w/ADA} & \textbf{w/TSA} & \textbf{\method (Full)} \\ \hline
CS          & 71.03 & 68.31 & \underline{89.72} & \textbf{98.19} \\
Facebook    & 77.89 & \underline{83.13} & 80.81 & \textbf{83.39} \\
ACM         & 51.57 & \underline{87.58} & 87.31 & \textbf{89.04} \\
Cora        & 88.23 & \underline{90.53} & 88.31 & \textbf{91.07} \\
T-Finance   & 69.84 & \underline{72.37} & 71.93 & \textbf{75.75} \\
CiteSeer    & 92.41 & 91.02 & \underline{93.53} & \textbf{94.12} \\
BlogCatalog & 49.91 & \underline{76.07} &{53.30} & \textbf{77.91} \\
Elliptic    & 32.24 & 70.47 & \underline{73.84} & \textbf{74.77} \\
Amazon      & 78.67 & \underline{81.37} & 81.23 & \textbf{82.67} \\
DGraph-Fin  & 52.47 & \underline{53.44} & 53.41 & \textbf{54.53} \\
Photo       & 73.21 & \underline{73.51} & 74.00 & \textbf{76.10} \\
Reddit      & 57.32 & 57.32 & \underline{57.86} & \textbf{59.69} \\
Weibo       & 88.13 & \underline{90.10} & 89.41 & \textbf{90.57} \\ \hline
\end{tabular}%
}
\end{table}

As shown in Table~\ref{tab: ablation}, both modules provide complementary gains.
\textbf{ADA} notably improves performance on structure-dominated datasets (e.g., ACM, Facebook, Amazon), validating its effectiveness in modeling disassortative structural patterns.
\textbf{TSA} yields consistent gains on feature-centric datasets (e.g., Cora, CiteSeer, Photo) by adapting scores at test time.
Combining both modules (\textbf{Full}) achieves the best results across all datasets, with prominent gains on ACM (+9.74\%) and CS (+17.29\%),
demonstrating the necessity of jointly modeling structural disassortativity and testing-time adaptability.

\subsection{Parameter Sensitivity Analysis}
The model involves a hyperparameter—the voting threshold $K$ in Eq.~(\ref{eq: voting}), which controls pseudo-label strictness.

\begin{figure}[!htbp]
    \centering
    \includegraphics[width=0.99\linewidth]{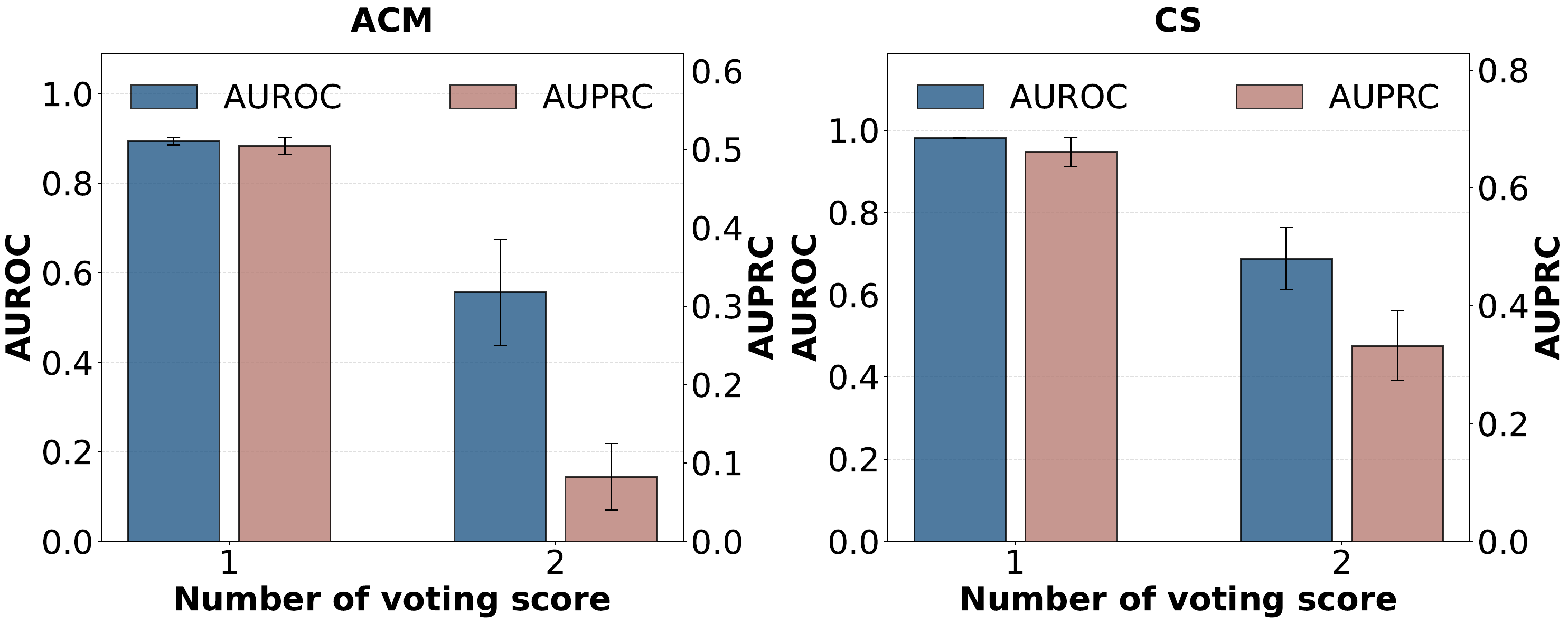}
    \caption{Effect of voting threshold $K$ on AUROC and AUPRC for ACM and CS datasets.}
    \label{fig: parameter_sensitivity}
\end{figure}
As shown in Fig.~\ref{fig: parameter_sensitivity}, a smaller $K$ ($K{=}1$) applies a lenient rule and consistently yields higher AUROC and AUPRC on ACM and CS, whereas a larger $K$ degrades performance by overly restricting pseudo-labeling. Thus, a lower threshold better balances precision and recall for robust testing-time adaptation.

\subsection{Experimental Analysis of \texorpdfstring{$\mathcal{AD}$}{AD}}
\label{sec: analysis_of_AD}
We further analyze the effectiveness of \method under anomaly disassortativity by computing $\mathcal{ND}$, $\mathcal{SD}$, and the normalized adaptability metric $\mathcal{AD^*}$ for each target domain. 
\begin{table}[H]
\centering
\caption{The correlation analysis between $\Delta$ AUROC  and $\mathcal{AD}$ compared with previous SOTA methods. Normalized $\mathcal{AD^*}$ derived from ${\mathcal{AD}}$ via smoothed normalization.}
\label{tab: AD_star_metric}
\resizebox{\linewidth}{!}{%
\begin{tabular}{c|ccc|c}
\hline
\textbf{Target Dataset} & $\mathop{\mathcal{ND}}$ & $\mathop{\mathcal{SD}}$ & $\mathcal{AD^*}$ & $\Delta$ AUROC \\ \hline
Cora        & 0.7928 & 0.3174 & 0.9800 & \textcolor{red}{$\uparrow$3.62}  \\
CiteSeer    & 0.7797 & 0.3313 & 0.9425 & \textcolor{red}{$\uparrow$3.17}  \\
BlogCatalog & 0.3682 & 0.7796 & 0.7548 & \textcolor{red}{$\uparrow$3.15}  \\
Elliptic    & 0.6645 & 0.3179 & 0.7747 & \textcolor{red}{$\uparrow$2.36}  \\
Facebook    & 0.8054 & 0.4653 & 0.7220 & \textcolor{red}{$\uparrow$14.78} \\
CS          & 0.7281 & 0.4635 & 0.6062 & \textcolor{red}{$\uparrow$15.73} \\
Photo       & 0.5849 & 0.8283 & 0.5323 & \textcolor{red}{$\uparrow$0.68}  \\
DGraph-Fin  & 0.4937 & 0.7431 & 0.5632 & \textcolor{red}{$\uparrow$1.46}  \\
ACM         & 0.6224 & 0.3960 & 0.6066 & \textcolor{red}{$\uparrow$8.90}  \\
Amazon      & 0.5853 & 0.8100 & 0.5091 & \textcolor{red}{$\uparrow$2.00}  \\
TFinance    & 0.5432 & 0.7302 & 0.4667 & \textcolor{red}{$\uparrow$3.56}  \\
Weibo       & 0.7269 & 0.6605 & 0.0454 & \textcolor{green}{$\downarrow$2.30} \\
Reddit      & 0.7906 & 0.8323 & 0.0200 & \textcolor{green}{$\downarrow$0.35} \\ \hline
\end{tabular}%
}
\end{table}

As shown in Table~\ref{tab: AD_star_metric}, higher $\mathcal{AD^*}$ corresponds to stronger cross-domain disassortativity and larger AUROC gains (e.g., Cora, CiteSeer). In contrast, low-disassortativity datasets (e.g., Weibo and Reddit) show limited or negative gains, suggesting reduced benefit when cross-domain discrepancies are weak. Overall, the positive correlation between $\mathcal{AD^*}$ and $\Delta$ AUROC validates the effectiveness of \method in adapting to pronounced feature–structure inconsistencies.

\subsection{Impact of \texorpdfstring{$\mathcal{ND}$ and $\mathcal{SD}$}{ND and SD}}
\label{sec: analysis ND and SD}

To further investigate the influence of $\mathcal{AD}$ on source and target domains, we present the $\mathcal{ND}$ and $\mathcal{SD}$ between different target domains with source domains, including PubMed, Flickr, Questions, and YelpChi. 
Specifically, $\mathcal{ND}$ quantifies the discrepancy in node feature distributions or semantics across domains, while $\mathcal{SD}$ measures the variability in graph connectivity patterns.  
Thus, larger values of $\mathcal{ND}$ and $\mathcal{SD}$ indicate greater source–target mismatch (domain shift) and are associated with degraded anomaly detection performance.  For clearer presentation, $\mathcal{SD}$ is highlighted in red, and $\mathcal{ND}$ is shown in blue.    
\begin{table}[H]
\centering
\caption{The AUROC of three models with different node disassortativity ($\mathcal{ND}$) and structure disassortativity ($\mathcal{SD}$) on cross-domain generalization.}
\label{tab: impact_nd_sd}
\resizebox{\linewidth}{!}{%
\begin{tabular}{c|ccccc}
\hline
Source Dataset & \multicolumn{5}{c}{PubMed, Flickr, Questions, YelpChi}                             \\ \hline
Target Dataset & ACM     & BlogCatalog    & Facebook       & CS             & T-Finance      \\ \hline
${\mathcal{ND}}$   & \textcolor{red}{0.6224}         & \textcolor{blue}{0.3682}         & \textcolor{red}{0.8054}         & \textcolor{red}{0.7281}         & \textcolor{blue}{0.5432}         \\
${\mathcal{SD}}$            & \textcolor{blue}{0.3960}         & \textcolor{red}{0.7796}         & \textcolor{blue}{0.4653}         & \textcolor{blue}{0.4635}         & \textcolor{red}{0.7302}         \\ \hline
ARC            & \ul{79.88}    & \ul {74.76}    & \ul{67.56}    & \ul{82.46}    & 64.10          \\
AnomalyGFM     & 60.79          & 57.77          & 58.64          & 67.86          & \ul{67.57}    \\
\method        & \textbf{89.04} & \textbf{77.91} & \textbf{83.39} & \textbf{98.19} & \textbf{76.29} \\ \hline
\end{tabular}%
}
\end{table}
Table~\ref{tab: impact_nd_sd} reports the AUROC of three GGAD models with different node structure disassortativity across domains. \textbf{ARC} performs particularly well on datasets with high $\mathcal{ND}$ values (e.g., Facebook, ACM, CS), suggesting its strong capability in mitigating feature-space discrepancies through node-level alignment. Conversely, \textbf{AnomalyGFM} excels on domains characterized by high $\mathcal{SD}$, such as T-Finance, where robust graph-structural modeling better generalizes across diverse topologies. In summary, \method consistently achieves the best performance across all targets by jointly bridging semantic and structural gaps, enabling robust cross-domain graph anomaly detection under both $\mathcal{ND}$ and $\mathcal{SD}$ disassortativity.

\section{Conclusion}

This study tackles cross-domain graph anomaly detection by identifying the \textbf{Anomaly Disassortativity} issue ($\mathcal{AD}$)—a feature and structure mismatch under domain shift that hinders model generalization.  
We introduce \method, a graph foundation model that achieves robust cross-domain adaptation through a single training phase.  
Experiments on thirteen real-world graphs confirm its superior generalization and state-of-the-art performance.  
The proposed $\mathcal{AD}$ framework offers a new perspective on domain shift in graph anomaly detection and lays a foundation for advancing generalist graph anomaly detection.  
In the future, we will extend this paradigm to graph-level anomaly detection, providing a more comprehensive understanding of anomalies across different domains.

\bibliographystyle{named}
\bibliography{ijcai26}

\appendix
\input{appendix}





\end{document}

%% file: appendix.tex
\appendix
\section{Theoretical Foundations of \texorpdfstring{$\mathcal{ND}$ and $\mathcal{SD}$}{ND and SD}}
To measure the divergence between source and target node distributions with high-order and low-order anomaly score, we adopt the square root of the Jensen–Shannon divergence (JSD), which yields a proper distance metric:
\begin{equation}
\label{eq:ND}
\mathop{\mathcal{ND}}\limits_{S \rightarrow T}
= \sqrt{\tfrac{1}{2} D_{\mathrm{KL}}\!\Big(\widehat{f}^{\text{S}}(x) \,\big\|\, M^{\text{node}} \Big) 
+ \tfrac{1}{2} D_{\mathrm{KL}}\!\Big(\widehat{f}^{\text{T}}(x) \,\big\|\, M^{\text{node}} \Big)},
\end{equation}
\begin{equation}
\label{eq:SD}
\mathop{\mathcal{SD}}\limits_{S \rightarrow T}
= \sqrt{\tfrac{1}{2} D_{\mathrm{KL}}\!\Big(\widebar{f}^{\text{S}}(x) \,\big\|\, M^{\text{str}} \Big) 
+ \tfrac{1}{2} D_{\mathrm{KL}}\!\Big(\widebar{f}^{\text{T}}(x) \,\big\|\, M^{\text{str}} \Big)},
\end{equation}
where $M^{\text{node}}=\frac{1}{2}\Big(\widehat{f}^{\text{S}}(x)+\widehat{f}^{\text{T}}(x)\Big)$ and $M^{\text{str}}=\frac{1}{2}\Big(\widebar{f}^{\text{S}}(x)+\widebar{f}^{\text{T}}(x)\Big)$ are the mixed distribution of the density function.

Since both $\mathcal{ND}$ and $\mathcal{SD}$ are constructed following an identical formulation, as the square root of the Jensen-Shannon divergence between two distributions (source and target) and their shared mixture, therefore we only provide theoretical analysis for the key properties of $\mathcal{ND}$, while $\mathcal{SD}$ follows the same logic analogously. 
\begin{theorem}
(Lower bound)~The anomaly score $\mathcal{ND}$ is lower-bounded by 0.  
\end{theorem}
\begin{proof}
Given any two distributions $P$ and $Q$, their KL divergence is non-negative, that is, $KL(P||Q) \geq 0$ with equality if and only if $P=Q$. Thus, their JS divergence $JS(P||Q)=\frac{1}{2} KL(P||M)+\frac{1}{2} KL(Q||M) \geq 0$ with equality if and only if $P=Q$. Since $\sqrt{JS(P||Q)}\geq 0$, it can be concluded that $\mathcal{ND}$ has the lower bound $0$. 
\end{proof}

\begin{theorem}
(Upper bound)~The anomaly score $\mathcal{ND}$ is upper-bounded by $1$.  
\end{theorem}
\begin{proof}
Given any two distributions $P$ and $Q$ with density fucntions $p(x)$ and $q(x)$, and their mixed distribution $M$ has the density fucntion $m(x)$, we have $m(x)=\frac{1}{2}\Big(p(x)+q(x)\Big)$. Hence, the KL divergence of $P$ and $M$ can be written as
\begin{align*}
KL(P||M)&=\int p(x)log\frac{p(x)}{m(x)}dx\\
&=\int p(x)log\frac{2p(x)}{p(x)+q(x)}dx\\
&=\int p(x)log2dx + \int p(x)log\frac{p(x)}{p(x)+q(x)}dx.
\end{align*}

Since the density function $p(x)$ satisfies $\int p(x)dx=1$, thus the first part $\int p(x)log2dx$ can be simplified as
\begin{align*}
\int p(x)log2dx=log2 \cdot\int p(x)dx=log2 \cdot 1=log2.
\end{align*}

By the concavity of the log function $f(t)=logt$, we have
\begin{align*}
\int p(x)f(t(x))dx \leq f \Big(\int p(x)t(x)dx\Big)
\end{align*}
where $p(x)$ satisfies $\int p(x)dx=1$.
Since the density function is non-negative, we have $p(x)+q(x)\geq p(x)$. It follows that $\frac{p(x)}{p(x)+q(x)}\leq 1$. Then, let $t(x)=\frac{p(x)}{p(x)+q(x)}$, $log(t(x))\leq 0$ when $t(x)\leq 1$. Thus, we conclude that
\begin{align*}
\int p(x)log\frac{p(x)}{p(x)+q(x)}dx &= \int p(x)log(t(x)))dx\\ 
&\leq \int p(x) \cdot 0 dx\\
&\leq 0
\end{align*}

Therefore, we have
\begin{align*}
KL(P||M)&=\int p(x)log2dx + \int p(x)log\frac{p(x)}{p(x)+q(x)}dx\\
&\leq log2+0\\
&\leq log2.
\end{align*}

In conclusion, we have proven that 
\begin{align*}
\sqrt{JS(P||Q)}&=\sqrt{\frac{1}{2}KL(P||M)+\frac{1}{2}KL(Q||M)}\\
&\leq \sqrt{\frac{1}{2} log2 +\frac{1}{2}log2}\\
&\leq \sqrt{log2}.
\end{align*}
If base-2 logarithms are used, the upper bound of $\sqrt{JS(P||Q)}$ is $1$. Besides, the upper bound holds only when the distribution $M$ is a convex combination of $P$ and $Q$, which is satisfied in our model. Thus, $\mathcal{ND}$ is upper-bounded by $1$. 
\end{proof}

\begin{theorem}
(Symmetry)~For any two domains $S$ and $T$, the $\mathcal{ND}(S,T)$ is symmetric.
\end{theorem}
\begin{proof}
\begin{align*}
\mathop{\mathcal{ND}}\limits_{S \rightarrow T}&= \sqrt{\frac{1}{2} D_{\mathrm{KL}}\Big(\widehat{f}^{\text{S}}(x) \,\big\|\, M^{\text{node}}\Big) +\frac{1}{2} D_{\mathrm{KL}}\Big(\widehat{f}^{\text{T}}(x) \,\big\|\, M^{\text{node}}\Big)} \\
&=\sqrt{\frac{1}{2} D_{\mathrm{KL}}\Big(\widehat{f}^{\text{T}}(x) \,\big\|\, M^{\text{node}}\Big) +\frac{1}{2} D_{\mathrm{KL}}\Big(\widehat{f}^{\text{S}}(x) \,\big\|\, M^{\text{node}}\Big)} \\
&=\mathop{\mathcal{ND}}\limits_{T \rightarrow S}.
\end{align*}
Thus, we have proven that $\mathcal{ND}$ is symmetric.
\end{proof}

\begin{theorem}
(Non-negative)~For any two domains $S$ and $T$, the $\mathcal{ND}$ is non-negative.
\end{theorem}
\begin{proof}
Since the JS divergence is non-negative, it can be concluded that $\mathcal{ND}=\sqrt{JS(P||Q)}$ is non-negative.
\end{proof}




\section{Description of Datasets}
\label{appendix: dataDes}

\noindent Following ARC~\cite{arc_liu2024arc}, we evaluate our model on 13 benchmark datasets grouped into four categories:
\textbf{(1)} citation networks with injected anomalies, 
\textbf{(2)} social networks with injected anomalies, 
\textbf{(3)} social networks with real anomalies, and 
\textbf{(4)} co-review and co-purchase networks with real anomalies.  
For each group, the largest dataset is used as the source for training, and the others serve as target domains for testing, ensuring a thorough assessment of \method’s generalization capability.

Table~\ref{tab: dsets_statis} summarizes dataset statistics. The selected datasets span academic, social, financial, and e-commerce domains, covering both synthetic and real anomalies to provide a comprehensive benchmark for cross-domain detection.

\begin{table*}[ht!]
\caption{Statistics of the seventeen datasets used for training and testing.} 
\centering
\label{tab: dsets_statis}
\resizebox{0.85\textwidth}{!}{%
\begin{tabular}{ccccccccc}
\hline
\multicolumn{1}{c|}{Dataset} & Train & \multicolumn{1}{c|}{Test} & \#Nodes & \#Edges & \#Features & Avg. Degree & \#Anomaly & \%Anomaly \\ \hline
\rowcolor{Gray}
\multicolumn{9}{c}{Citation networks with injected anomalies}                                                                                   \\
\multicolumn{1}{c|}{Cora}        & -            & \multicolumn{1}{c|}{$\checkmark$} & 2,708     & 5,429      & 1,433  & 3.90  & 150    & 5.53  \\
\multicolumn{1}{c|}{CiteSeer}    & -            & \multicolumn{1}{c|}{$\checkmark$} & 3,327     & 4,732      & 3,703  & 2.77  & 150    & 4.50  \\
\multicolumn{1}{c|}{ACM}         & -            & \multicolumn{1}{c|}{$\checkmark$} & 16,484    & 71,980     & 8,337  & 8.73  & 597    & 3.62  \\
\multicolumn{1}{c|}{PubMed}      & $\checkmark$ & \multicolumn{1}{c|}{-}            & 19,717    & 44,338     & 500    & 4.50  & 600    & 3.04  \\ \hline
\rowcolor{Gray}
\multicolumn{9}{c}{Social networks with injected anomalies}                                                                                     \\
\multicolumn{1}{c|}{BlogCatalog} & -            & \multicolumn{1}{c|}{$\checkmark$} & 5,196     & 171,743    & 8,189  & 66.11 & 300    & 5.77  \\
\multicolumn{1}{c|}{Flickr}      & $\checkmark$ & \multicolumn{1}{c|}{-}            & 7,575     & 239,738    & 12,047 & 63.30 & 450    & 5.94  \\ \hline
\rowcolor{Gray}
\multicolumn{9}{c}{Social networks with real anomalies}                                                                                         \\
\multicolumn{1}{c|}{Facebook}    & -            & \multicolumn{1}{c|}{$\checkmark$} & 1,081     & 55,104     & 576    & 50.97 & 25     & 2.31  \\
\multicolumn{1}{c|}{Weibo}       & -            & \multicolumn{1}{c|}{$\checkmark$} & 8,405     & 407,963    & 400    & 48.53 & 868    & 10.30 \\
\multicolumn{1}{c|}{Reddit}      & -            & \multicolumn{1}{c|}{$\checkmark$} & 10,984    & 168,016    & 64     & 15.30 & 366    & 3.33  \\
\multicolumn{1}{c|}{Questions}   & $\checkmark$ & \multicolumn{1}{c|}{-}            & 48,921    & 153,540    & 301    & 3.13  & 1,460  & 2.98  \\ \hline
\rowcolor{Gray}
\multicolumn{9}{c}{Co-review and Co-purchase networks with real anomalies}                                                                      \\
\multicolumn{1}{c|}{CS}          & -            & \multicolumn{1}{c|}{$\checkmark$} & 18,333     & 186,019     & 8      & 10.14 & 600  & 3.27 \\
\multicolumn{1}{c|}{Photo}       & -            & \multicolumn{1}{c|}{$\checkmark$} & 7,650     & 248,731    & 745    & 32.51 & 450    & 5.88   \\
\multicolumn{1}{c|}{Amazon}      & -            & \multicolumn{1}{c|}{$\checkmark$} & 10,244    & 175,608    & 25     & 17.18 & 693    & 6.76  \\
\multicolumn{1}{c|}{YelpChi}     & $\checkmark$ & \multicolumn{1}{c|}{-}            & 23,831    & 49,315     & 32     & 2.07  & 1,217  & 5.10  \\ \hline
\rowcolor{Gray}
\multicolumn{9}{c}{Financial networks with real anomalies}                                                                                     \\
\multicolumn{1}{c|}{Elliptic}    & -            & \multicolumn{1}{c|}{$\checkmark$} & 46,564    & 73,248     & 93     & 1.57  & 4,545  & 9.76   \\
\multicolumn{1}{c|}{T-Finance}   & -            & \multicolumn{1}{c|}{$\checkmark$} & 39,357    & 42,445,086  & 10     & 1,078.46 & 1,803  & 4.58   \\
\multicolumn{1}{c|}{DGraph-Fin}  & -            & \multicolumn{1}{c|}{$\checkmark$} & 3,700,550 & 73,105,508 & 17     & 2.16  & 15,509 & 1.30   \\ \hline
\end{tabular}%
}
\end{table*}

\begin{itemize}
    \item \textbf{Cora}, \textbf{CiteSeer}, \textbf{PubMed}~\cite{sen2008collective}, and \textbf{ACM}~\cite{tang2008arnetminer}: citation graphs where nodes denote papers, edges represent citations, and node features are bag-of-words vectors.
    
    \item \textbf{BlogCatalog} and \textbf{Flickr}~\cite{dominant_ding2019deep}: social graphs in which nodes are users and edges denote friendships or follows. Node attributes derive from user-generated text and tags.
    
    \item \textbf{Amazon} and \textbf{YelpChi}~\cite{rayana2015collective,mcauley2013amateurs}: review-based graphs for opinion fraud detection. We adopt the Amazon-UPU and YelpChi-RUR variants, where edges link users or reviews sharing contextual relations.
    
    \item \textbf{Facebook}~\cite{xu2022contrastive}: a social network capturing user friendships and interactions.
    
    \item \textbf{Reddit}~\cite{kumar2019predicting}: a forum network where users are nodes and edges reflect shared discussions; banned users are labeled as anomalies.
    
    \item \textbf{Weibo}~\cite{kumar2019predicting}: a microblog network where suspicious posting behaviors within short intervals indicate anomalies.
    
    \item \textbf{Questions}~\cite{platonov2023critical}: a Q\&A interaction graph from Yandex Q, with user embeddings derived from profile text.
    
    \item \textbf{Elliptic}~\cite{ijcai2022p270_elliptic}: a Bitcoin transaction graph where nodes are transactions and edges represent money flow.
    
    \item \textbf{Photo}~\cite{mcauley2015image_photo}: an Amazon co-purchase graph with products as nodes and textual features from user reviews.
    
    \item \textbf{T-Finance}~\cite{tang2022rethinking}: a financial transaction network where nodes represent accounts; fraudulent activities are labeled as anomalies.
    
    \item \textbf{DGraph-Fin}~\cite{huang2022dgraph}: a large-scale financial graph linking user accounts via emergency-contact relations, with profile-based node features.
    
    \item \textbf{CS}~\cite{shchur2018pitfalls}: a co-authorship network from the Microsoft Academic Graph.
\end{itemize}

\section{Details of Implementation}
\label{appendix: implementation}

\noindent\textbf{Baseline Implementation.}  
All methods, including \method and baselines, are trained on the source dataset $\mathcal{T}_{\text{train}}$ with full anomaly labels and evaluated on each target dataset $\mathcal{T}_{\text{test}}$ without retraining.  
Node features are projected into a $d_u = 64$ latent space using PCA.  Baseline results are reproduced from official implementations with tuned hyperparameters.  All experiments share a unified training and evaluation pipeline for fair comparison.

\noindent\textbf{Metrics.}  
Following~\cite{tang2024gadbench,arc_liu2024arc}, we use two standard metrics: AUROC and AUPRC, where higher values indicate better performance.  
All results are averaged over five runs with standard deviations reported.

\noindent\textbf{Implementation Details.} 
The experiments in this study were conducted on a Linux server running Ubuntu 20.04. The server was equipped with a 13th Gen Intel(R) Core(TM) i7-12700 CPU, 64GB of RAM, and an NVIDIA GeForce RTX A6000 GPU (48GB memory). For software, we used Anaconda3 to manage the Python environment and PyCharm as the development IDE. The specific software versions were Python 3.8.14, CUDA 11.7, DGL 0.9.1~\cite{wang2019deep} and PyTorch 2.0.1~\cite{paszke2019pytorch}.

\section{Algorithmic description}
\label{appendix: algo_des}

The algorithmic description of the training and inference process of \method is summarized in Algorithm.~\ref{alg: train_alg}, and Algorithm.~\ref{alg: infer_alg}, respectively.

\begin{algorithm}[t]
\caption{The Training algorithm of \method}
\label{alg: train_alg}
\renewcommand{\algorithmicrequire}{\textbf{Input:}}
\renewcommand{\algorithmicensure}{\textbf{Parameters:}}
\begin{algorithmic}[1]
\Require Training datasets $\mathcal{T}_{train}$. 
\Ensure Number of epoch $T$;
\State Generate initial parameters for all learnable parameters. 


\For{\text{\rm each epoch} $t=0,1,2,...,T$}
    \For{$\mathcal{D} \in \mathcal{T}_{train}$}
        \State Get $\mathbf{H}^{[1]},\cdots,\mathbf{H}^{[l]} \gets$ high-order embedding via Eq.~(\ref{eq: X_prop_trans})
        
        \State Get $\mathbf{R} \gets$ calculate residual of $\mathbf{H}^{[l]}$ and $\mathbf{H}^{[1]}$ via Eq.~(\ref{eq: residual_information})
        
        \State Get $\bar{\mathbf{H}}$ and $\hat{\mathbf{H}} \gets$ affinity encoder via Eq.~(\ref{eq: gcn_structure_information}) and Eq.~(\ref{eq: mlp_structure_information})
        
        \State Calculate loss $\mathcal{L}_{High}$  and $\mathcal{L}_{Low}$ via Eq.~(\ref{eq: loss_high}) and Eq.~(\ref{eq: loss_low})
        \State optimization model parameters via Eq.~(\ref{eq: loss_final}).
    \EndFor 
\EndFor
\end{algorithmic}
\end{algorithm}

\begin{algorithm}[!h]
\caption{The Inference algorithm of \method}
\label{alg: infer_alg}
\renewcommand{\algorithmicrequire}{\textbf{Input:}}
\renewcommand{\algorithmicensure}{\textbf{Parameters:}}
\begin{algorithmic}[1]
\Require Test datasets $\mathcal{T}_{test}$. 
\Ensure Pre-trained \method model;
\For{$\mathcal{D} \in \mathcal{T}_{test}$}

    \State Get $\mathbf{H}^{[1]},\cdots,\mathbf{H}^{[l]} \gets$ high-order embedding via Eq.~(\ref{eq: X_prop_trans})

    \State Get $\mathbf{R} \gets$ calculate residual of $\mathbf{H}^{[l]}$ and $\mathbf{H}^{[1]}$ via Eq.~(\ref{eq: residual_information})

    \State Get $\mathcal{RS}(v_i) \gets$ high-order anomaly scoring via Eq.~(\ref{eq: RS_score}).

    \State Get $\bar{\mathbf{H}}$ and $\hat{\mathbf{H}} \gets$ affinity encoder via Eq.~(\ref{eq: gcn_structure_information}) and Eq.~(\ref{eq: mlp_structure_information})
     
    \State Get $\mathcal{AS}(v_i) \gets$ low-order anomaly scoring via Eq.~(\ref{eq: affinity score}).

    \State Get $\mathcal{S}_{\mathcal{AD}}(v_i) \gets$ anomaly disassortativity  adapter via Eq.~(\ref{eq: adas_weight}) - Eq.~(\ref{eq: fused_score}).
    
    \State Get $\mathbf{\hat{Y}} \gets$  pseudo-label voter via Eq.~(\ref{eq: pseudo_label}) - Eq.~(\ref{eq: voting}).
    
    \State Calculate final anomaly  score $\mathcal{S}(v_i) \gets$ via Eq.~(\ref{eq: adaptive_score}) - Eq.~(\ref{eq: adaptive loss}). 
\EndFor
\end{algorithmic}
\end{algorithm}

\begin{table*}[!htbp]
\renewcommand{\arraystretch}{1.2}
\setlength{\extrarowheight}{3pt}  
\centering
\caption{Anomaly detection performance in terms of AUPRC (in percent, mean±std). Highlighted are the results ranked \textcolor{firstcolor}{\textbf{\underline{first}}}, \textcolor{secondcolor}{\underline{second}}, and \textcolor{thirdcolor}{\underline{third}}.} 
\label{tab: main_aurpc}
\resizebox{\textwidth}{!}{%
\begin{tabular}{c|ccccccccccccc|l}
\hline
\textbf{Method}  & \textbf{cs} & \textbf{Facebook} & \textbf{ACM} & \textbf{Cora} & \textbf{T-Finance} & \textbf{CiteSeer} & \textbf{BlogCatalog} & \textbf{elliptic} & \textbf{Amazon} & \textbf{DGraph-Fin} & \textbf{photo} & \textbf{Reddit} & \textbf{Weibo} & \multicolumn{1}{c}{\textbf{rank}} \\ \hline
\rowcolor{Gray} \multicolumn{15}{c}{\textbf{GAD methods}} \\
GCN(2017)        & 2.88±0.31 & 1.59±0.11 & 5.27±1.12 & 7.41±1.55 & 7.81±5.33 & 6.40±1.40 & 7.44±1.07 & \third{15.25±17.70} & 6.96±2.04 & 0.95±1.81 & 4.82±0.29 & 3.39±0.39 & \third{67.21±15.20} & 12.38 \\
GAT(2018)        & 4.33±0.10 & 3.14±0.37 & 4.70±0.75 & 6.49±0.84 & 4.93±2.81 & 5.58±0.62 & 12.81±2.08 & 6.17±0.26 & 15.74±17.85 & OOM & 6.04±0.04 & 3.73±0.54 & 33.34±9.80 & 11.58 \\
BGNN(2021)       & 2.34±0.02 & 3.81±2.12 & 3.48±1.33 & 4.90±1.27 & 3.56±0.09 & 3.91±1.01 & 5.73±1.47 & 10.78±1.01 & 7.51±0.58 & OOM & 4.74±0.01 & 3.52±0.50 & 30.26±29.98 & 14.58 \\
BWGNN(2022)      & 5.44±2.97 & 2.54±0.63 & 7.14±0.20 & 7.25±0.80 & 11.31±2.90 & 6.35±0.73 & 8.99±1.12 & 13.79±3.65 & 13.12±11.82 & 1.31±0.27 & 7.41±2.12 & 3.69±0.81 & 12.13±0.71 & 10.62 \\
GHRN(2023)       & 15.40±1.74 & 2.41±0.62 & 5.61±0.71 & 9.56±2.40 & 9.36±5.61 & 7.79±2.01 & 10.94±2.56 & 11.37±1.26 & 7.54±2.01 & 1.33±0.05 & 6.86±1.24 & 3.24±0.33 & 28.53±7.38 & 9.92 \\
DOMINANT(2019)   & 9.80±0.14 & 2.95±0.06 & 15.59±2.69 & \third{12.75±0.71} & 13.71±0.58 & 13.85±2.34 & \third{35.22±0.87} & 5.97±0.79 & 6.11±0.29 & OOM & 7.43±0.74 & 3.49±0.44 & \first{81.47±0.22} & 7.67 \\
CoLA(2021)       & 14.13±0.89 & 1.90±0.68 & 7.31±1.45 & 11.41±3.51 & 9.42±1.48 & 8.33±3.73 & 6.04±0.56 & 9.36±6.46 & 11.06±4.45 & OOM & 4.42±2.83 & 3.71±0.67 & 7.59±3.26 & 11.42 \\
HCM-A(2022)      & 6.00±2.21 & 2.08±0.60 & 4.01±0.61 & 5.78±0.76 & OOM & 4.18±0.75 & 6.89±0.34 & OOM & 5.87±0.07 & OOM & 6.80±0.42 & 3.18±0.23 & 21.91±11.78 & 15 \\
TAM(2023)        & 17.34±1.56 & \third{8.40±0.97} & 23.20±2.36 & 11.18±0.75 & 4.82±5.61 & 11.55±0.44 & 10.57±1.17 & 6.13±1.06 & 10.75±3.10 & OOM & 9.97±2.47 & 3.94±0.13 & 16.46±0.09 & 8.25 \\
CAGAD(2024)      & 6.35±3.43 & 2.61±0.76 & 7.97±4.67 & 5.31±3.20 & 13.57±6.55 & 3.85±1.60 & 6.40±3.06 & \second{16.24±3.75} & 3.49±0.73 & \third{1.35±0.12} & 7.21±0.56 & \first{13.56±18.91} & 20.95±18.34 & 10.42 \\
S-GAD(2024)       & 17.62±1.47 & 1.82±0.41 & 6.71±1.87 & 5.07±2.10 & 4.40±3.43 & \third{34.34±0.46} & 31.82±6.85 & 6.65±2.77 & 5.43±4.62 & OOM & 3.92±1.42 & 4.01±0.58 & 62.85±7.55 & 11.17 \\
CONSISGAD(2024)  & 9.04±1.51 & 2.72±1.65 & 13.42±3.77 & 6.69±1.24 & 5.00±1.34 & 7.19±2.34 & 15.83±4.00 & 11.59±6.47 & \third{30.15±7.30} & 1.34±0.19 & 6.36±1.70 & 3.57±0.25 & 14.53±7.99 & 9.31 \\
SmoothGNN(2025)  & 7.66±4.59 & 3.76±2.38 & 10.69±4.07 & 7.55±2.17 & 13.63±14.53 & 5.30±1.77 & 16.52±5.64 & 11.53±3.56 & 26.07±16.29 & \second{1.42±0.32} & \third{10.61±12.29} & 3.44±0.23 & 17.57±13.31 & 8.15 \\
SpaceGNN(2025)   & 4.52±0.51 & 1.65±0.26 & 7.47±1.19 & 4.78±1.12 & OOM & 4.03±0.35 & 8.54±1.29 & 11.98±1.62 & 5.76±1.38 & 1.27±0.18 & 4.15±0.56 & 3.89±0.55 & 17.98±11.19 & 13.58 \\
GCTAM(2025)      & 15.22±1.04 & \second{9.61±0.96} & \second{48.09±0.28} & 9.52±0.69 & 4.17±1.59 & 10.29±0.28 & 27.47±0.54 & 6.10±0.13 & 13.71±0.11 & OOM & 8.74±0.22 & 4.34±0.17 & 16.87±1.31 & \third{7.67} \\ \hline
\rowcolor{Gray} \multicolumn{15}{c}{\textbf{GGAD methods}} \\
ARC(2024)        & \second{36.45±0.24} & 8.38±2.39 & \third{40.62±0.10} & \second{49.33±1.64} & \second{20.70±3.63} & \second{45.77±1.25} & \first{36.06±0.18} & 6.75±0.05 & \second{44.25±7.41} & 1.18±0.09 & \first{28.93±1.89} & 4.48±0.28 & 64.18±0.55 & \second{3.69} \\
UNPrompt(2025)  & \third{19.47±0.86} & 2.61±0.45 & 10.45±1.55 & 6.02±0.20 & 1.33±0.17 & 4.47±0.32 & 24.89±3.25 & 6.79±3.81 & 10.27±7.04 & 1.23±0.52 & 6.33±1.17 & \second{5.15±0.65} & 18.67±4.33 & 10.12 \\
AnomalyGFM(2025) & 6.21±0.28 & 5.48±3.43 & 5.41±0.69 & 8.65±1.06 & \third{16.72±1.98} & 7.11±0.97 & 10.44±3.38 & OOM & 11.02±5.59 & OOM & 9.46±0.75 & \third{5.03±0.39} & 37.15±16.41 & 8.09 \\ \hline
\method          & \first{65.80±0.86} & \first{17.53±4.42} & \first{49.86±0.76} & \first{57.06±2.54} & \first{25.42±0.48} & \first{58.26±1.91} & \second{35.84±0.19} & \first{18.83±0.22} & \first{49.15±3.19} & \first{1.48±0.06} & \second{25.22±2.67} & 4.05±0.42 & \second{70.18±0.26} & \first{1.62} \\ \hline
\end{tabular}%
}
\end{table*}

\section{More Experiments: Performance Comparison of AUPRC}
\label{appendix: auprc_result}

Table~\ref{tab: main_aurpc} reports the anomaly detection performance in terms of AUPRC across thirteen datasets. \method achieves the \textbf{best overall ranking (1.62)}, ranking \textbf{first on seven} datasets (ACM, Facebook, Amazon, Cora, CiteSeer, cs, T-Finance) and \textbf{second on four} (BlogCatalog, Weibo, photo, DGraph-Fin), demonstrating strong adaptability to diverse semantic and structural shifts.  
Compared to the strongest baseline ARC, \method shows notable gains, e.g., \textbf{+9.24\%} on ACM, \textbf{+4.90\%} on Amazon, and \textbf{+5.67\%} on Facebook. While ARC remains competitive on some datasets, \method consistently achieves top performance across conventional and financial graphs. Traditional models (GCN, GAT, BGNN) perform poorly on complex datasets, underscoring their limited generalization capabilities.  
These results validate the effectiveness of \method's fusion of invariant semantic features and structure-aware affinity encoding for robust zero-shot anomaly detection across heterogeneous domains.